\documentclass[letterpaper]{article} % DO NOT CHANGE THIS
\usepackage{aaai23}  % DO NOT CHANGE THIS
\usepackage{times}  % DO NOT CHANGE THIS
\usepackage{helvet}  % DO NOT CHANGE THIS
\usepackage{courier}  % DO NOT CHANGE THIS
\usepackage[hyphens]{url}  % DO NOT CHANGE THIS
\usepackage{graphicx} % DO NOT CHANGE THIS
\usepackage{natbib}  % DO NOT CHANGE THIS AND DO NOT ADD ANY OPTIONS TO IT
\usepackage{caption} % DO NOT CHANGE THIS AND DO NOT ADD ANY OPTIONS TO IT
\usepackage{algorithm}
\usepackage{algorithmic}

\usepackage{amsmath}
\usepackage{amssymb}
\usepackage{amsfonts}
\usepackage{amsthm}
\usepackage{subfig}
\usepackage{longtable}
\usepackage{multirow}
\usepackage{appendix}
\graphicspath{{fig_IGD/}}
\newtheorem{lemma}{Lemma}%[section]
\newtheorem{assumption}{Assumption}
\newtheorem{theorem}{Theorem}

\newcommand{\be}{\begin{equation}}
\newcommand{\ee}{\end{equation}}
\newcommand{\ba}{\begin{array}}
\newcommand{\ea}{\end{array}}
\newcommand{\bea}{\begin{eqnarray}}
\newcommand{\eea}{\end{eqnarray}}
\newcommand{\beas}{\begin{eqnarray*}}
\newcommand{\eeas}{\end{eqnarray*}}
\newcommand{\p}{\partial}
\newcommand{\x}{\mathbf{x}}
\newcommand{\y}{\mathbf{y}}
\newcommand{\W}{\mathbf{W}}
\newcommand{\w}{\mathbf{w}}
\newcommand{\mH}{\mathbf{H}}
\newcommand{\inn}[1]{\left< #1 \right>}
\newcommand{\norm}[1]{\left\|#1\right\|}
\newcommand{\f}[2]{\frac{#1}{#2}}

\newcommand{\R}{\mathbb{R}}
\newcommand{\Og}{\Omega}
\renewcommand{\u}{\mathbf{u}}
\renewcommand{\v}{\mathbf{v}}
\title{Implicit Stochastic Gradient Descent for Training Physics-informed Neural Networks}
\author {
    Ye Li, Song-Can Chen, Sheng-Jun Huang
}
\affiliations {
    College of Computer Science and Technology/Artificial Intelligence,
       Nanjing University of Aeronautics and Astronautics\\
       MIIT Key Laboratory of Pattern Analysis and Machine Intelligence, Nanjing, China\\
    yeli20@nuaa.edu.cn, s.chen@nuaa.edu.cn, huangsj@nuaa.edu.cn
    }

\begin{document}

\maketitle

\begin{abstract}
Physics-informed neural networks (PINNs) have effectively been demonstrated in solving forward and inverse differential equation problems,
but they are still trapped in training failures when the target functions to be approximated exhibit high-frequency or multi-scale features.
In this paper, we propose to employ implicit stochastic gradient descent (ISGD) method to train PINNs for improving the stability of training process.
We heuristically analyze how ISGD overcome stiffness in the gradient flow dynamics of PINNs, especially for problems with multi-scale solutions.
We theoretically prove that for two-layer fully connected neural networks with large hidden nodes, randomly initialized ISGD converges to a globally optimal solution for the quadratic loss function.
Empirical results demonstrate that ISGD works well in practice and compares favorably to other gradient-based optimization methods such as SGD and Adam, while can also effectively address the numerical stiffness in training dynamics via gradient descent.
\end{abstract}

\section{Introduction}
Gradient descent (GD) and practical stochastic gradient descent with mini-batch gradients (SGD)
are widely used optimization algorithms, especially in optimizing deep neural networks.
Formally, the goal of optimization is to find a weight vector $\hat{\theta}$ in parameter space
$\R^m$ that minimizes a loss $L(\theta)$.
The GD algorithm is the updating procedure of model weights in the direction of the steepest loss gradient:
\be\label{gd}
\theta_{n+1} = \theta_{n} - \alpha \cdot \nabla L(\theta_n),
\ee
where $\alpha$ is the learning rate.
The SGD replaces the gradient $\nabla L(\theta)$ with a mini-batch gradient $\nabla \hat{L}_i(\theta)$, where $\hat{L}_i$ is the loss computed on mini-batch data instead of the whole dataset.
The continuous gradient flow is defined as a curvature $\theta(t)$ that satisfies the following ordinary differential equation (ODE):
\be\label{gd_ode}
\frac{d}{dt}\theta(t) = -\nabla_{\theta} L(\theta(t)).
\ee
It is easy to show that when the learning rate is sufficiently small,
the discrete updates $\{\theta_n\}_{n=0}^{\infty}$ computed by Eq.\eqref{gd} stay close to a function $\{\theta(t_n)\}_{n=0}^{\infty}$ where $t_n=n\alpha$.
Variants based on GD/SGD,
such as AdaGrad \citep{Adagrad}, RMSMprop \citep{RMSProp}, and Adam \citep{Adam}, have been developed in recent years.

Despite its numerous successes in practical optimization tasks such as optimizing deep neural networks,
GD/SGD may suffer from numerical instability in some key hyperparameters, such as the learning rate and batch size. For example, if the learning rate is misspecified, GD/SGD may numerically diverge, and the model training fails.
The main reason is the \emph{stiffness} in the gradient flow dynamics.
Typically, the gradient flow dynamics is called a \emph{stiff ODE} when the gap between the maximum and minimum eigenvalues of the Hessian matrix is large \citep{Wang20a}.
We can simply perform a linearization for the gradient flow \eqref{gd_ode} and obtain
\be
\frac{d}{dt}\tilde{\theta}(t) = -\nabla^2_{\theta} L(\tilde{\theta}(t))\cdot \tilde{\theta}(t).
\ee
The largest eigenvalue of the Hessian dictates the fastest time-scale of the ODEs.
%The continuous gradient flow can be considered as ordinary differential equations (ODEs), and GD is the discretization of ODEs.
%In the language of numerical analysis, GD is the explicit Euler method in numerically solving the gradient flow problem, and the learning rate corresponds to the stepsize.
In the language of numerical analysis, to ensure the numerical stability of GD, we need $\alpha\leq 2/\lambda_{\max}(\nabla^2_{\theta}L(\theta))$,
where $\lambda_{\max}(\nabla^2_{\theta}L(\theta))$ is the maximum eigenvalue of the Hessian matrix \citep{butcher2016numerical}.

From the theory of numerical analysis, GD/SGD is not suitable for stiff ODEs,
because a very small learning rate and very large number of iterations are required to maintain numerical stability.
One of the outstanding first-order solvers with strong stability for stiff ODEs is the implicit (backward) Euler method:
\be\label{sec-eq1}
\theta_{n+1} = \theta_{n} - \alpha \cdot \nabla L(\theta_{n+1}),
\ee
where a large learning rate can be used.
Eq.\eqref{sec-eq1} is also known as the implicit gradient descent (IGD) or implicit stochastic gradient descent (ISGD) method, as the next iteration $\theta_{n+1}$ appears implicitly on the right side of Eq.\eqref{sec-eq1}, and cannot be computed explicitly.

\begin{figure*}[tbp]
   \centering
     \subfloat[smooth solution $u_L(x)=\sin(2\pi x)$]
     {\includegraphics[width=0.23\textwidth]{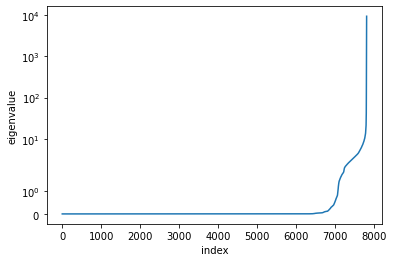}
     \includegraphics[width=0.23\textwidth]{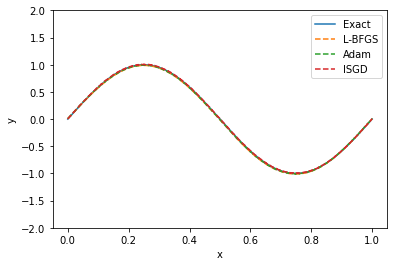}}
     \qquad\subfloat[multi-scale solution $u_H(x)=\sin(2\pi x)+0.1\sin(50\pi x)$]
     {\includegraphics[width=0.23\textwidth]{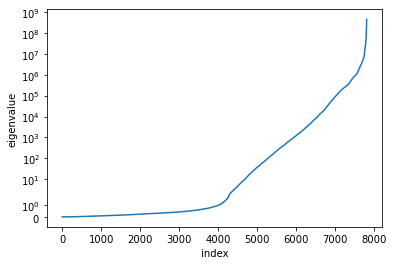}
     \includegraphics[width=0.23\textwidth]{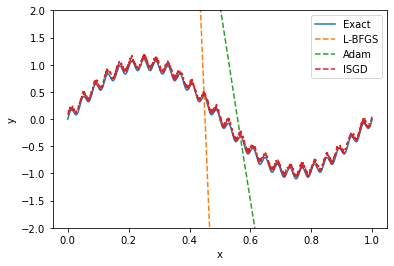}}
     \caption{\textit{1D Poisson equation}: Results for the heuristic example in Section \ref{heuristic} obtained by training a conventional PINN (5-layer, 200 hidden units, $\tanh$ activations) with gradient descent based Adam optimizer, quasi-Newton based L-BFGS optimizer and our ISGD optimizer. All eigenvalues of $\nabla^2_{\theta}L(\theta)$ are computed and arranged in increasing order.
     (a): smooth solution $u_L(x)=\sin(2\pi x)$, the maximum eigenvalue is 1.1e+04, non-stiff and all three optimizers trained well. (b): solution with multi-scale features $u_H(x)=\sin(2\pi x)+0.1\sin(50\pi x)$, the maximum eigenvalue is 4.6e+08, stiff and both Adam and L-BFGS failed to train, while our ISGD trained well.}\label{fig:heuristic-example}
\end{figure*}

%While the implicit Euler scheme is a well-known first-order method with sound stability in the theory of numerical methods for ODEs, it is less studied in the machine learning area.
%Physics-informed neural networks (PINNs) have lately received great attention thanks to their flexibility in tackling a wide range of forward and inverse problems involving partial differential equations.
Physics-informed neural networks (PINNs) are neural networks with outputs constrained to approximately satisfy a system of partial differential equations (PDEs) by using a regularization functional $\mathcal{R}(u_{\theta}(\mathbf{x}))$ that typically represents the residual of PDEs.
A general loss function representation of PINNs takes the form
\begin{equation}
  L(\theta) = \f{1}{N_u}\sum_{i=1}^{N_u}|\mathbf{u}_i-u_{\theta}(\mathbf{x}_i)|^2+
  \mathcal{R}(u_{\theta}(\mathbf{x})),
\end{equation}
where the given set input output pairs $(\mathbf{x}_i,\mathbf{u}_i)$ are corresponding to the initial/boundary conditions of PDEs.
The most popular optimizers for training PINNs are gradient descent based Adam optimizer and quasi-Newton based L-BFGS optimizer \citep{DeepXDE21}.
Howerver, the additional regularization term $\mathcal{R}(u_{\theta}(\mathbf{x}))$ has been shown to \emph{increase the stiffness} of the gradient dynamics \cite{Wang20a}, causing model training failures especially when the target functions to be approximated exhibit high-frequency or multi-scale features.
%They showed that for smooth PDE solutions L-BFGS can find a good solution with fewer iterations than Adam, because L-BFGS uses second-order derivatives of the loss function, while Adam relies only on first-order derivatives.
Typically, for stiff solutions, L-BFGS is more likely to be stuck at a bad local minimum, and Adam may need very small learning rate and very large number of iterations.
%We believe that the implicit scheme instead of explicit gradient descent is a suitable way to overcome the instability in the PINNs training.
We claim that IGD/ISGD is more stable than GD/SGD and L-BFGS in the PINNs training when fitting multi-scale solutions.
As an example, Figure \ref{fig:heuristic-example} contrasts these approaches.
As the solution of Poisson equation changes from smooth to multi-scale, the maximum eigenvalue of Hessian increases significantly and the gradient flow dynamics of PINN becomes stiff, both Adam and L-BFGS become divergent while our IGD/ISGD is still convergent.
%showing that IGD/ISGD is convergent while Adam and L-BFGS are divergent for PINNs approximation of Poisson equation with multi-scale solutions.

\subsection{Contributions}
%In this paper, we theoretically analyze the training dynamics of the IGD/ISGD method and show the global convergence properties.
%We also empirically show the strong stability of our IGD/ISGD method when applied to training PINNs as well as standard deep learning benchmark problems.
Our main contributions can be summarized in the following points:
\begin{itemize}
  \item We first propose to employ the IGD/ISGD method to train PINNs. We theoretically and numerically show that IGD/ISGD can overcome the stiffness in the gradient flow dynamics of PINNs, especially for PDEs with multi-scale solutions.
  \item We used practical L-BFGS and Adam optimizer to deal with the implicit updates in IGD/ISGD, which is effective in practice. The computational cost is comparable to Adam. Furthermore, the method is stable for the learning rate and batch size, making it easier for nonexperts to process neural network training tasks.
  \item The IGD global convergence property is proven. We theoretically prove that for two-layer fully connected neural networks with large hidden nodes, randomly initialized IGD converges to a globally optimal solution at a linear convergence rate for the quadratic loss function.
\end{itemize}

\subsection{Related Work}
\textbf{Gradient Descent. }
%There is a growing body of research on the theoretical analysis of GD/SGD on an over-parameterized neural network.
Global convergence of gradient descent based methods have been proved when training deep neural network despite the objective function being non-convex \citep{Du19,Du19D,Du18,ALS19,ZCZ20}.
%They proved that as the width of the neural network is large enough, the randomly initialized gradient descent converges to zero training loss at a linear rate.
The dynamics of neural network weights under GD converge to a point that is close to the minimum norm solution under proper conditions \citep{SS21}.% subject to the condition that there is no training error when using the linear approximation to the neural network.

Toulis and his collaborators \citep{ToulisAiroldi17,Toulis16,Toulis14} first theoretically studied the implicit stochastic gradient descent algorithm,
and claimed it to be more stable than standard stochastic gradient descent.
%They derived the asymptotic statistical efficiency of SGD procedures (both explicit and implicit) and showed that the efficiency loss of SGD procedures can be well quantified.
%They suggested that implicit stochastic gradient descent procedures are poised to become a workhorse for approximate inference from large datasets.
However, both the theoretical and practical results of them are only suited to generalized linear models.
%The fixed-point computation of the implicit update may not be feasible for some nonlinear models, especially for models with spectral bias phenomenon or stiffness phenomenon.
%Inspired by the implicit Euler method for solving numerical ODE problems,
The implicit scheme was extended to combine with the ResNet architecture with implicit Euler skip connections (called IE-ResNet) by \citep{BE20b} to improve the robustness and generalization ability.
%and showed that it can largely improve the robustness and generalization ability under adversarial attacks when compared with vanilla ResNet of the same parameter size.
%They treat the implicit scheme in every iteration as a nonlinear least square problem, use the gradient descent method to approximate it, and provide a practical method to solve the implicit scheme.
%Despite the pioneering work of \citet{ToulisAiroldi17},
%\citet{BE18} independently proposed an implicit gradient descent algorithm for the k-means clustering problem.
%Fixed-point iteration was also performed to solve the implicit gradient step.
%However, the fixed-point iteration requires that the gradient of the loss function has a bounded Lipschitz constant, which is hard for some other practical problems with stiffness phenomenon.
The IGD/ISGD method was also applied to optimize the k-means clustering problem \citep{BE18} and the objective matrix factorization loss function that appears in recommendation systems \citep{BE20a}, and the convergence time was effectively improved.

\textbf{PINNs. }
%Some data that appear in computational science and engineering are fitted to satisfy physical laws, mathematically expressed via PDEs.
%Traditional finite difference methods and finite element methods are used to solve PDEs numerically,
%However, the computational burden grows heavier as a more accurate solution is expected.
With the rapid growth of deep learning,
using neural networks to represent PDE solutions has attracted the attention of many researchers.
Based on the early studies of \citet{PU92,LLF98},
\citet{PINN19} proposed the pioneering work of PINNs to solve both forward and inverse problems involving nonlinear PDEs.
PINNs have demonstrated remarkable power in applications including fluid dynamics \citep{RYK20,JCLK21,MJK20}, biomedical engineering \citep{SYP20}, meta-material design \citep{FZ19,CLK20}, software packages \citep{DeepXDE21}, and numerical simulators \citep{SimNet20,CWW21}.
Adaptive activation functions can be applied to accelerate PINN training \citep{JKK20b,JKK20a,JSKK}.
However, despite early empirical success, the original formulations of PINNs  often struggles to handle problems exhibiting high-frequency and multi-scale behavior.

Recent works by \citet{Wang20a,Wang20b,Wang20c} have identified two fundamental weaknesses in conventional PINN formulations.
The first is the remarkable discrepancy in the convergence rate between the data-based loss function and the physical-based loss function.
%This can be repaired by adaptively adjusting the weights between different terms in the loss function.
The second is related to the spectral bias, which indeed exists in PINN models and is the leading reason that prevents them from accurately approximating high-frequency or multi-scale functions.
%Multiple Fourier feature embeddings can be a good choice for better approximating multi-scale functions.
In fact, they demonstrated that the gradient flow of PINN models becomes increasingly stiff for PDE solutions exhibiting high-frequency or multi-scale behavior.
This result motivates us to use robust implicit numerical schemes such as IGD/ISGD for the numerical solution to the gradient flow of PINN models.

\subsection{Organization of the paper}
In Section 2, we present the methodology of the proposed IGD/ISGD method.
The PINNs framework is also introduced briefly for completeness.
Two heuristic examples are presented to show the strong stability of the IGD/ISGD method.
In Section 3, we analyze the training dynamics of the IGD/ISGD method when applied to neural network training tasks.
Some technical proofs are given in the Appendix.
In Section 4, we report various computational examples for inferring the solution of ordinary/partial differential equations by PINNs.
Additional computational examples for regression and classification problems are given in the Appendix.
Finally, we conclude in Section 5 with a summary.

\section{Methodology}

\subsection{Physics-informed neural networks}
%In this section, we first briefly introduce the framework of PINNs.
PINNs are neural networks that imbeds differential equations into neural network training.
%A neural network is constructed to approximate the solution of differential equations.
The initial/boundary condition data of the differential equations are treated as the supervised learning component in the objective loss function, while the residual of the differential equations is applied as an unsupervised regularization factor in the objective loss function.
%PINN is an efficient method for solving forward and inverse differential equation problems.
%PINN can easily incorporate all the given information, such as governing equations and initial/boundary condition experimental data, into the loss function, and the original problem is approximated by an optimization problem.
%One feature of PINNs is that it uses automatic differentiation to represent all the differential operators; hence, there is no explicit need for mesh generation.
%The PINN algorithm aims to learn a surrogate $u_{\theta}$ for predicting the solution $u$ of the governing PDE.
We consider a parametrized PDE system given by:
\beas
\begin{array}{l}
  \mathcal{F}(\x,u,u_x,\ldots,\lambda) = 0, \quad \x\in\Og, \\
  u(\x) = g_0(\x),\quad \x\in \p\Og,
\end{array}
\eeas
where $\x$ are the spatial and time coordinates, $u=u(\x)$ is the solution to the PDE with boundary/initial data $g_0(\x)$, $\mathcal{F}$ denotes the PDE residual, and $\lambda$ is the PDE parameter.
For example, $\mathcal{F}=-u_{xx}-f(x)=0$ is the simplest 1D Poisson equation for a given function $f(x)$.
The vanilla PINN uses a fully connected feed-forward neural network $u_{\theta}(x)$ to approximate the solution $u(x)$ by minimizing the following loss function:
\be\label{pinn-loss1}
L(\theta) = \omega_d L_{data} + \omega_f L_{PDE},
\ee
where
\beas%\label{pinn-loss2}
L_{data} &=& \f{1}{N_d}\sum_{j=1}^{N_d}|u_{\theta}(\x_d^j)-g_0(\x_d^j)|^2 ,\\
L_{PDE} &=& \f{1}{N_f}\sum_{i=1}^{N_f}|\mathcal{F}(\x_f^i)|^2.
\eeas
Here, $\{\x_d^j\}_{j=1}^{N_d}$ represents the training data points on $\p\Og$ while $\{\x_f^i\}_{i=1}^{N_f}$ represents the set of residual points in $\Og$.
%The training data points may include additional supervised data in $\Og$ for some inverse problems.
$\omega_f$ and $\omega_d$ are the user-specified weighting coefficients for different loss terms.
The first term $L_{data}$ includes the known boundary/initial conditions and experimental data, which is the usual supervised data-driven part of the neural network.
To compute the residuals in the loss function, automatic differentiation is applied to compute the derivatives of the solution with respect to the independent variables.
This constitutes the physics-informed part of the neural network as given by the second term $L_{PDE}$.

The resulting optimization problem is to find the minimum of the loss function by optimizing the trainable parameters $\theta$.
Gradient descent based first-order optimizers such as SGD and Adam \citep{Adam},
or quasi-Newton based optimizers like L-BFGS \citep{L-BFGS},
are widely used in PINNs training.
However, as \citet{Wang20b} claimed, ``...PINNs using fully connected architectures often fail to achieve stable training and produce accurate predictions, especially when the underlying PDE solutions contain high-frequencies or multi-scale features".
The gradient flow dynamics of PINNs will become stiff as multi-scale phenomena appear, so explicit GD based optimizers may be unstable, and L-BFGS is more likely to be stuck at a bad local minimum.
%The numerical stiffness in the gradient flow dynamics of PINNs is also theoretically analyzed and numerically examined by them.
%The reason is that PINN exhibits stiffness in the gradient flow dynamics or spectral bias in the neural network approximation.
As we mentioned in the previous section, implicit schemes like IGD/ISGD are more stable to overcome the stiffness problems.
Two illustrative examples are presented to show the robustness of IGD/ISGD in the next section.

%Most recent work in the realm of PINN has been to mitigate issues of this sort by introducing modifiable weights to these loss functions, thus altering the loss function during training to arrive at a reasonably good approximation.Some work also attempts to adaptively change the collocation points or their weights.In this paper, we originally study the training dynamics of PINN and investigate stable optimization algorithms for training the PINN as well as other deep neural networks. The method is introduced in the next section.

\subsection{Heuristic examples with stability}\label{heuristic}
In this section, we present two heuristic examples to show the stability of IGD/ISGD and the instability of GD/IGD.

\textbf{Analytical stiff problem. }The first example is to theoretically analysis the learning rate constraint in the gradient flow dynamics of stiffness problems.
We denote a fabricated loss function by
\be
L(\theta_1,\theta_2) = \f{K_1}{2}(\theta_1-\theta_1^*)^2+\f{K_2}{2}(\theta_2-\theta_2^*)^2,\nonumber
\ee
where $\theta_i\in\mathbb{R},\;i=1,2$ are two parameters to be optimized,
$K_i>0,\;i=1,2$ are two constants.
The eigenvalues of the Hessian matrix of $L(\theta_1,\theta_2)$ are characterized by $K_1$ and $K_2$.
When $K_1$ and $K_2$ differ in scales, for example, $K_1=10^{-4}$ and $K_2=10^4$,
the gradient flow of the loss function suffers from the stiffness phenomenon.

A direct computation shows that the loss function update procedure of GD has the following relation:
\bea
\frac{L(\theta_1^{n+1},\theta_2^{n+1})}{L(\theta_1^{n},\theta_2^{n})}
\leq \max\{(1-\alpha K_1)^2,(1-\alpha K_2)^2\}.
\eea
%The number $D=\max\{(1-\alpha K_1)^2,(1-\alpha K_2)^2\}$ represents the loss decay rate in every adjacent iteration.
Typically, we need $D=\max\{(1-\alpha K_1)^2,(1-\alpha K_2)^2\}\leq 1$ to guarantee loss decay, which implies $\alpha\leq \f{2}{\max\{K_1,K_2\}}$.
%When $K_1\ll K_2$, $\alpha$ will be very small and $D$ very close to 1, meaning that the loss decays very slowly.
When $K_1=10^{-4}$ and $K_2=10^4$, we have $\alpha\leq 10^{-4}$ and $D\leq 1-10^{-8}$, meaning that the loss decays very slowly, and very large number of iterations (at least $\mathcal{O}(10^8)$) are needed to converge.
For a large learning rate $\alpha$, the loss decay rate $D$ may be greater than 1, and the loss may increase as the iterations increase, causing numerical instability in the gradient flow dynamics computation.
%This is a simple example to show the instability of GD with respect to the learning rate.

For IGD method, the loss function update procedure has the following relation:
\bea
\frac{L(\theta_1^{n+1},\theta_2^{n+1})}{L(\theta_1^{n},\theta_2^{n})}
\leq \max\{\f{1}{(1+\alpha K_1)^2},\f{1}{(1+\alpha K_2)^2}\}.
\eea
The loss decay rate $D=\max\{\f{1}{(1+\alpha K_1)^2},\f{1}{(1+\alpha K_2)^2}\}$
satisfies $D< 1$ automatically for all learning rates $\alpha>0$ and regardless of the scales of $K_1,K_2$, and $D$ is even smaller for larger $\alpha$.
%The loss always decays as iterations increase for all $\alpha>0$ and $K_1,K_2>0$.
This shows the strong stability of IGD to deal with stiffness phenomena.

\textbf{1D Poisson equation with multi-scale solution. }This heuristic example is to show the advantage of IGD/ISGD when the gradient flow dynamics of PINN is stiff. We consider a simple 1D Poisson equation
\begin{equation}\label{1dPoisson}
  -\Delta u(x) = f(x),\quad x\in(0,1)
\end{equation}
subject to the boundary condition
\begin{equation*}
  u(0)=u(1)=0.
\end{equation*}
We consider two fabricated solutions: one is $u_L(x)=\sin(2\pi x)$ exhibiting low frequency on the whole domain, and another is $u_H(x)=\sin(2\pi x)+0.1\sin(50\pi x)$ exhibiting low frequency in the macro-scale and high frequency in the micro-scale.
%$f(x)$ can be derived using Eq.\eqref{1dPoisson}.
Though this example is simple and pedagogical, it resembles many practical scenarios with multi-scale phenomenons.

We represent the unknown solution $u(x)$ by a 5-layer fully-connected neural network $u_{\theta}(x)$ with 200 units per hidden layer.
%The parameters of the network can be learned by minimizing the following loss function
%\beas
%L(\theta) &=& L_{data}(\theta)+L_{PDE}(\theta) \\
%&=& \frac{1}{2}\left(|u_{\theta}(0)|^2+|u_{\theta}(1)|^2\right) \\
%& & +\frac{1}{N_r}\sum_{i=1}^{N_r}|-\Delta u_{\theta}(x_i)-f(x_i)|^2.
%\eeas
$N_r=1000$ training points $\{x_i,f(x_i)\}$ are uniformly sampled in the interval $(0,1)$.
%It has been theoretically analyzed and numerically examined that the gradient flow dynamics of PINN approximation to the solution of Eq.\eqref{1dPoisson} is non-stiff for $u_L(x)$ while stiff for $u_H(x)$.
Figure \ref{fig:heuristic-example} shows the results obtained by training PINN with gradient descent based Adam optimizer \citep{Adam} with default settings for a maximum $10^7$ epochs, quasi-Newton based L-BFGS optimizer \citep{L-BFGS} with default settings, and our ISGD method with learning rate 0.1 for a maximum $10^4$ epochs.
We observe that all three optimizers can train PINN well for smooth solution $u_L(x)$ when there is non-stiff.
As multi-scale solution $u_H(x)$ appears, the maximum eigenvalue of Hessian has a significant rise from 1.1e+04 to 4.6e+08.
The gradient flow dynamics of PINN becomes stiff, and the popular Adam optimizer is incapable of training PINN to the correct solution even after a million training epochs.
The L-BFGS optimizer is also failed to train.
As a comparison, our ISGD method can train PINN well both for smooth $u_L(x)$ as well as multi-scale $u_H(x)$ with larger learning rate and smaller iterations.

\subsection{Loss decay of GD/IGD}
\citet{Wang20a} shows that the loss decay of GD is
\bea
  & & L(\theta_{n+1})-L(\theta_n) \nonumber\\
  &=& \alpha \norm{\nabla_{\theta}L(\theta_n)}_2^2
  \left(-1+\f{1}{2}\alpha\sum_{i=1}^{N}\lambda_iy_i^2\right),
\eea
where $\lambda_1\leq \lambda_2\leq \cdots\leq\lambda_N$ are eigenvalues of the Hessian matrix $\nabla_{\theta}^2L(\xi)$,
%$\xi=t\theta_n+(1-t)\theta_{n+1}$ for some $t\in(0,1)$,
and $\mathbf{y}=(y_1,...,y_N)$ is a normalized vector.
When $\{\theta_n\}_{n=0}^{\infty}$ reaches a local or global minimum,
the Hessian matrix $\nabla_{\theta}^2L(\xi)$ is semi-positive definite and all $\lambda_i\geq 0$ for all $i=1,...,N$.
Moreover, for the multi-scale solution $u_H(x)$, computational results show that many eigenvalues of $\nabla_{\theta}^2L(\xi)$ are very large (see Figure \ref{fig:heuristic-example}), i.e., stiff during gradient flow dynamics.
As a result, it is very possible that $L(\theta_{n+1})-L(\theta_n)>0$, which implies that the GD method fails to decrease the loss.
A similar computation approach (see the Appendix) shows that the loss decay of IGD is
\bea
  & & L(\theta_{n+1})-L(\theta_n) \nonumber\\
  &=& \alpha \norm{\nabla_{\theta}L(\theta_{n+1})}_2^2
  \left(-1-\f{1}{2}\alpha\sum_{i=1}^{N}\lambda_iy_i^2\right),
\eea
means that the loss will always decay regardless of the stiffness of the gradient flow dynamics of PINNs.
In addition, the linear convergence rate of IGD is strictly proven in Section 3.

\subsection{Implementation of the IGD/ISGD method}
Although the IGD/ISGD method Eq.\eqref{sec-eq1} looks simple and theoretically stable, one difficulty that can not be ignored is the implicity of the nonlinear Eq.\eqref{sec-eq1}.
It can also be expressed as the celebrated proximal point algorithm \citep{BE18,PPA76}:
\be\label{sub-opt}
\theta_{n+1} = arg\min\limits_{\theta} \left\{ \f{1}{2}\norm{\theta-\theta_n}^2 + \alpha\cdot L(\theta) \right\}.
\ee
Hence, when $\alpha$ is sufficiently small,
$\theta_{n+1}$ is approximately close to its previous updates $\theta_n$ with the original loss as a regularizer.
This sub-optimization task requires additional computation and brings difficulties for the whole optimization process.

To reduce the computational burden, we take a practical ``ISGD,L-BFGS" (or ``ISGD,Adam") optimizer for PINNs training with multi-scale solutions.
Here ``ISGD,L-BFGS" means that we first use ISGD with large learning rate for a certain number of iterations, and then switch to L-BFGS with default settings.
In the sub-optimization problem \eqref{sub-opt}, we also apply L-BFGS to compute $\theta_{n+1}$.
The optimizer L-BFGS does not require learning rate, and the neural network is trained until convergence, so the number of iterations is also ignored for L-BFGS \citep{L-BFGS}.
%We mention before that L-BFGS optimizer is not suitable to train PINNs with stiff solutions directly, since it is more likely to be stuck at a bad local minimum.
Here, the successful application of L-BFGS in ``ISGD,L-BFGS" optimizer is that both the sub-optimization problem and the subsequent optimization problem have good initial point $\theta_n$, thus are easier for L-BFGS to achieve good convergence properties.
%The details are illustrated in Algorithm 1.
The ``ISGD,Adam" optimizer is to repalce L-BFGS by Adam optimizer with default settings in the ``ISGD,L-BFGS" optimizer when the parameters of PINNs are too large for the quasi-Hessian matrix computation.
The details are illustrated in Algorithm 1.
%We also note that when the parameters of PINNs are too large for the quasi-Hessian matrix computation in L-BFGS, the standard SGD or Adam optimizer instead of L-BFGS can be applied to solve the sub-optimization problem and the subsequent optimization problem.

\begin{algorithm}[tb]
\caption{Practical ``ISGD,Adam" optimization for the loss $L(\theta)$ with stiff solutions}
\label{alg:algorithm}
\textbf{Input}: initial $\theta_0$; ISGD learning rate $\alpha$ and maximum iterations $K_0$; the inner Adam learning rate $\gamma$ and maximum iterations $K_1$ ; the outer Adam learning rate $\eta$ and maximum iterations $K_2$ \\
%\textbf{Parameter}: Optional list of parameters\\
\textbf{Output}: the optimized $\theta^*$
\begin{algorithmic}[1] %[1] enables line numbers
\STATE Let $n=0$.
\WHILE{$n< K_0$}
\STATE Let $\tilde{\theta}_0 = \theta_n$ and $k=0$.
\WHILE{$k\leq K_1$}
\STATE Update $\tilde{\theta}_{k+1} = \textbf{Adam}\left(\left\{ \f{1}{2}\norm{\tilde{\theta}-\theta_n}^2 + \alpha\cdot L(\tilde{\theta}) \right\}|_{\tilde{\theta}=\tilde{\theta}_k},\gamma\right)$ and $k\leftarrow k+1$
\ENDWHILE
\STATE Update $\theta_{n+1}=\tilde{\theta}_{K_1}$ and $n\leftarrow n+1$
\ENDWHILE
\WHILE{$K_0\leq n< K_0+K_2$}
\STATE Update $\theta_{n+1}=\textbf{Adam}(L(\theta)|_{\theta=\theta_n},\eta)$ and $n\leftarrow n+1$
\ENDWHILE
\STATE Denote $\theta^*=\theta_{K_0+K_2}$
\STATE \textbf{return} the optimized $\theta^*$
\end{algorithmic}
\end{algorithm}

\section{Training dynamics analysis of IGD/ISGD}
In this section, we analyze the neural network training dynamics of our IGD/ISGD method. The technical proofs are given in the Appendix.

\textbf{Quadratic loss. }
We show that randomly initialized IGD method with a constant positive step size
converges to the global minimum at a linear rate.
For simplicity of proof, we demonstrate a two-layer neural network with the quadratic loss functions.
The global convergence property can be extended to an arbitrary $N$-layer neural network with quadratic loss with the technique introduced in \citet{Du19}.
Formally, we consider a neural network of the following form:
\be
u(\W,\mathbf{a},\x) = \f{1}{\sqrt{m}}\sum_{r=1}^{m}a_r\sigma(\w_r^T\x),
\ee
where $\x\in\mathbb{R}^d$ is the input data, $\w_r\in\mathbb{R}^d$ is the weight vector of the first layer,
$a_r \in \mathbb{R}$ is the weight vector of the output layer, and $\sigma(z)$ is the activation function.
We focus on the empirical risk minimization problem with a quadratic loss.
Given a training data set $\{(\x_i,y_i)\}_{i=1}^N$, we minimize
\be\label{l_igd}
L(\W,\mathbf{a}) = \sum_{i=1}^{N}\f{1}{2}|y_i-u(\W,\mathbf{a},\x_i)|^2.
\ee
For simplicity, we fix the second layer and apply the IGD method to optimize the first layer
\be\label{w_igd}
\W(n+1) = \W(n) - \alpha\f{\p L(\W(n+1),\mathbf{a})}{\p \W(n+1)},
\ee
where $\alpha>0$ is the learning rate.

The training dynamics of $u(\W(n),\mathbf{a},\x_i)$ strongly relies on the Gram matrix $\mH(n+1)$ defined by
\be\label{sec3-1}
\mH_{ij}(n+1)  = \sum_{r=1}^{m} \inn{\f{\p u_i(n+1)}{\p\w_r},\f{\p u_j(n+1)}{\p\w_r}},
\ee
and it's limit Gram matrix $\mH^{\infty}$ defined by
\be\label{sec3-3}
\mH^{\infty}_{ij} = \x_i^T\x_j\mathbb{E}_{\w\sim\mathcal{N}(\mathbf{0},\mathbf{I})}
\sigma^{'}(\w^T\x_i)\sigma^{'}(\w^T\x_j).
\ee

The positivity of $\mH^{\infty}$ is the key to prove convergence.
We first state some technical assumptions.
\begin{assumption}\label{ass1}
  The activation function $\sigma(\cdot)$ is smooth, analytic, and is not a polynomial function.
Moreover, both $\sigma(\cdot)$ and its derivatives are Lipschitz continuous, i.e., there exists a constant $C>0$ such that $|\sigma(0)|\leq C$ and for any $z_1,z_2\in\mathbb{R}$,
\beas
|\sigma(z_1)-\sigma(z_2)|\leq c|z_1-z_2|,\\
|\sigma^{'}(z_1)-\sigma^{'}(z_2)|\leq c|z_1-z_2|.
\eeas
\end{assumption}
%It is easy to verify that some common activation functions,
%such that $softplus: \sigma(z)=\log(1+\exp(z))$,
%$tanh: \sigma(z)=\f{e^z-e^{-z}}{e^z+e^{-z}}$,
%satisfy the above assumption with Lipschitz constant 1.
Here and below, we use the same constant $C$ without confusion for simplicity to represent different constants independent of $m,N,\lambda_{\min}(\mH^{\infty})$.

\begin{assumption}\label{ass2}
  No two input data are parallel, i.e., for any $i\neq j$, we need $\x_i\neq c\x_j$ for any constant $c$.
\end{assumption}

%These two assumptions are sufficient to prove the strict positivity of the Gram matrix and the training process stability of the IGD method.
Now we present our main theorem. The proofs in detail can be found in the Appendix.
\begin{theorem}\label{thm3}
Assume Assumption \ref{ass1} and Assumption \ref{ass2} hold and for all $i\in[N]$, $\norm{\x_i}_2\leq C$, $|y_i|\leq C$,
and the hidden numbers
$m\geq\max\left\{\f{16CN^2}{\lambda_{\min}(\mH^{\infty})^2}\log(\f{2N}{\delta}),
\f{16C^2N^4}{\lambda_{\min}(\mH^{\infty})^4}\right\}$,
and the learning rate $\alpha\leq\f{C\lambda_{\min}(\mH^{\infty})}{N^2}$ for some constant $C$,
and we i.i.d. initialize $\w_r\sim\mathcal{N}(\mathbf{0},\mathbf{I})$, $a_r\sim unif|{-1,1}|$ for $r\in[m]$, then with probability $1-\delta$ we have for $n=0,1,2,...$
\be\label{th3-eq1}
L(n) \leq \left(\f{1}{1+\f{\alpha\lambda_{\min}(\mH^{\infty})}{2}}\right)^n L(0).
\ee
where the quadratic loss $L(n)\doteq L(\W(n),\mathbf{a})$ is defined by Eq.\eqref{l_igd}.
\end{theorem}

\textbf{PINN loss. }
For the PINN loss Eq.\eqref{pinn-loss1},
it has been observed that the Gram matrix $\mH^{\infty}$ may not guarantee strict positivity
\citep[see][Figure 1]{Wang20b}, and the proof technique may fail.
However, as demonstrated in the next section, the convergence and strong stability of IGD/ISGD for training PINNs are numerically verified.

\section{Computational Results}
In this section, we compare the performance of SGD optimizer, Adam optimizer and our ISGD optimizer in training PINNs to solve different differential equations.
%The SGD and Adam are with default settings like learning rate scheduler.
The hyper-parameters used in the three optimizers are listed in Table 1.
We note \#Iterations = ($K_0\cdot K_1+K_2)\cdot batchs$, where $K_0,K_1,K_2$ are hyper-parameters in Algorithm 1.
The wall-clock computational time is proportional to \#Iterations, so the computational time is comparable for three optimizers in all numerical examples.
%We see that all three optimizers are computationally comparable.
More computational results are given in the Appendix.
\begin{table}[tbp]
\centering
\begin{tabular}{c|ccc}\hline
 Example & Optimizer & Learning rate & \#Iterations \\\hline
 4.1 & SGD(Adam) & 0.001 & 120,000 \\\cline{2-4}
 ($\epsilon=2$) & ISGD, Adam & 0.5, 0.001 & 102,000 \\\hline
 4.1 & SGD(Adam) & 0.001 & 400,000 \\\cline{2-4}
 ($\epsilon=0.01$) & ISGD, Adam & 0.5, 0.001 & 360,000 \\\hline
 \multirow{3}*{4.2} & SGD(Adam) & 0.0005 & 2,000,000 \\\cline{2-4}
  & ISGD, Adam & 0.5, 0.0005 & 1,100,000 \\\hline
 \multirow{3}*{4.3} & SGD(Adam) & 0.0005 & 1,000,000 \\\cline{2-4}
  & ISGD, Adam & 0.5, 0.0005 & 550,000 \\\hline
\end{tabular}
\caption{Hyper-parameters used in the three optimizers for the following 3 examples. ``SGD(Adam)" represents SGD shares the same hyper-parameters with Adam. ``ISGD, Adam" is referred in Algorithm 1.}
\end{table}

\subsection{PINN for ordinary differential equations}
\begin{figure*}[tbp]
   \centering
     \subfloat[learning rate = 0.001]{\includegraphics[width=0.28\textwidth]{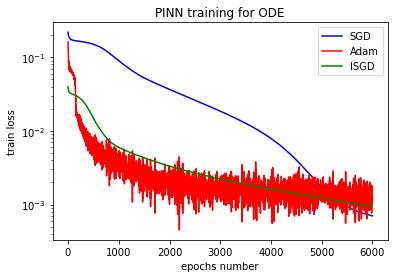}}
     \subfloat[learning rate = 0.5]{\includegraphics[width=0.28\textwidth]{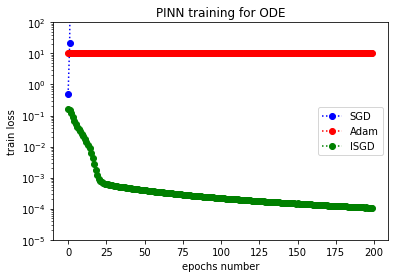}}
     \subfloat[learning rate = 0.5]{\includegraphics[width=0.28\textwidth]{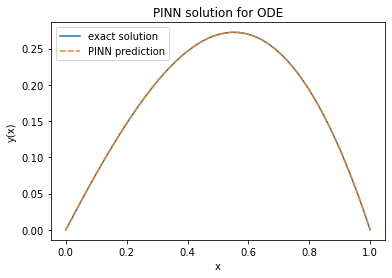}}\\
     \subfloat[learning rate = 0.001]{\includegraphics[width=0.28\textwidth]{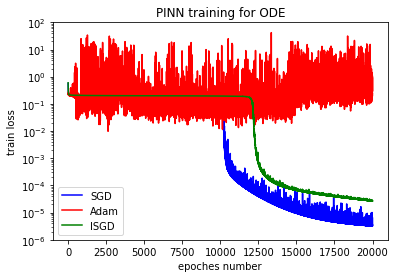}}
     \subfloat[learning rate = 0.5]{\includegraphics[width=0.28\textwidth]{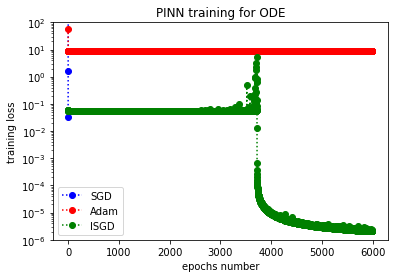}}
     \subfloat[learning rate = 0.5]{\includegraphics[width=0.28\textwidth]{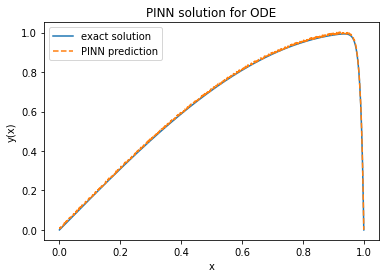}}
     \caption{The optimization training results for ODE Eq.\eqref{ex2-eq1}.
     Top row: $\epsilon=2.0$. Down row: $\epsilon=0.01$. }\label{fig:ex2-1}
\end{figure*}
Singularly perturbed ordinary differential equations have been successfully applied to
many fields including gas dynamics, chemical reaction, fluid mechanics, elasticity, etc.
To find the solution is a hot and difficult problem because it contains a very small parameter $\epsilon$.
We consider the second-order linear singularly perturbed boundary value differential equation
\bea\label{ex2-eq1}
\left\{
  \begin{array}{l}
  -\epsilon y^{''}(x) + y^{'}(x) = f(x), \quad x\in(0,1), \\
  y(0)=0,\quad y(1)=0.
  \end{array}
\right.
\eea
The true solution is chosen as $y(x)=\frac{1-e^{\frac{x}{\epsilon}}}{e^{\frac{1}{\epsilon}}-1} + \sin(\frac{\pi}{2}x)$, and $f(x)$ is given according to Eq.\eqref{ex2-eq1}.
$\epsilon>0$ is a constant; when $\epsilon$ is very small, a boundary layer exists near the boundary $x=1$. Let $y_{\theta}(x)$ be the neural network approximation of $y(x)$, then the PINN loss function can be defined as
\beas
L(\theta) &=& \frac{1}{2}\left[|y_{\theta}(0)-y(0)|^2+|y_{\theta}(1)-y(1)|^2 \right] \\
 &&+ \frac{1}{N}\sum_{i=1}^{N}\left|-\epsilon y_{\theta}^{''}(x_i) + y_{\theta}^{'}(x_i) - f(x_i)\right|^2.
\eeas

We choose $N=400$ randomly sampled points to compute the loss function, a batch size of 40 for a small learning rate $\alpha=0.001$, and a full batch size for a large learning rate $\alpha=0.5$.
A neural network with 4 hidden layers, every 50 units with \emph{tanh} activations, is applied in all the computations.
The results are shown in Figure \ref{fig:ex2-1}.
For the case $\epsilon=2$, the true solution is smooth.
As shown in Fig. \ref{fig:ex2-1}(a)(b), we find that the ISGD optimizer can significantly improve training convergence and remain stable for different learning rates.
For the case $\epsilon=0.01$, as shown in Fig. \ref{fig:ex2-1}(f), the true solution has a boundary layer near $x=1$, and the large gradient creates difficulties for the optimizers.
As shown in Fig. \ref{fig:ex2-1}(d)(e), more epochs and a smaller learning rate are required to be convergent for this singularity phenomenon. While the SGD and Adam optimizers are not convergent for large learning rates, the ISGD can still have stable convergent results, demonstrating the robustness of the proposed method.

\subsection{PINN for Poisson equation}
\begin{figure}[tbp]
   \centering
     \subfloat[learning rate = 0.0005]{\includegraphics[width=0.45\columnwidth]{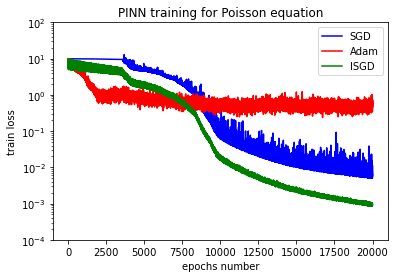}}
     \subfloat[learning rate = 0.5]{\includegraphics[width=0.45\columnwidth]{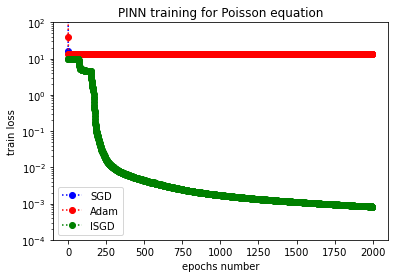}}\\
     \subfloat[prediction $u_{\theta}(x,y)$]{\includegraphics[width=0.45\columnwidth]{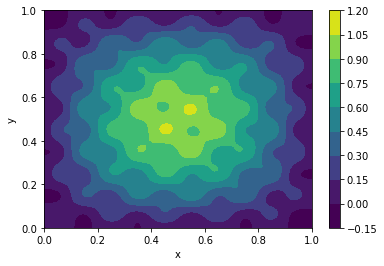}}\
     \subfloat[error $|u(x,y)-u_{\theta}(x,y)|$]{\includegraphics[width=0.45\columnwidth]{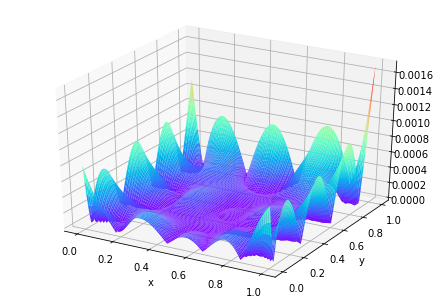}}
     \caption{PINN training for the Poisson equation \eqref{ex2-eq2}.}\label{fig:ex2-3}
\end{figure}
Poisson equation is an elliptic partial differential equation of broad utility in theoretical physics.
We consider the Poisson equation on the domain $\Og=[0,1]\times[0,1]$
\bea\label{ex2-eq2}
\left\{
  \begin{array}{l}
  -\frac{\p^2u}{\p x^2}-\frac{\p^2u}{\p y^2} = f(x,y),\quad (x,y)\in\Og, \\
  u(x,y) = 0, \quad (x,y)\in\p\Og.
  \end{array}
\right.
\eea
The true solution is chosen as $u(x,y)=\sin(\pi x)\sin(\pi y)+0.1\sin(10\pi x)\sin(10\pi y)$ with multi-scale features.
The PINN loss function is defined as
\beas
L(\theta)  = \frac{1}{N_b}\sum_{i=1}^{N_b}\left|u_{\theta}(x_i,y_i)-u(x_i,y_i) \right|^2 + \\ \frac{1}{N_f}\sum_{j=1}^{N_f}\left|\frac{\p^2u_{\theta}(x_j,y_j)}{\p x^2}+\frac{\p^2u_{\theta}(x_j,y_j)}{\p y^2} - f(x_j,y_j) \right|^2.
\eeas
We choose $N_b=400$ randomly sampled points on $\p\Og$,
and $N_f=4,000$ randomly sampled points in $\Og$ to compute the loss function.
A neural network with 6 hidden layers, every 100 units with \emph{tanh} activations,
is applied in all the computations.
The three optimizer training results for $\alpha=0.0005$ and $0.5$ are shown in Fig. \ref{fig:ex2-3}(a) and Fig. \ref{fig:ex2-3}(b), respectively.
We see that neither SGD nor Adam can train well as learning rate increases, but our ISGD trains well for different values of $\alpha$.
The PINN prediction is plotted in Fig. \ref{fig:ex2-3}(c),
and the absolute error is shown in Fig. \ref{fig:ex2-3}(d), with an absolute error less than $0.2\%$.
We see that the PINN trained by the ISGD optimizer can obtain stable and accurate results for the Poisson equation \eqref{ex2-eq2}.

\subsection{PINN for Helmholtz equation}
\begin{figure}[tbp]
   \centering
     \subfloat[learning rate = 0.0005]{\includegraphics[width=0.45\columnwidth]{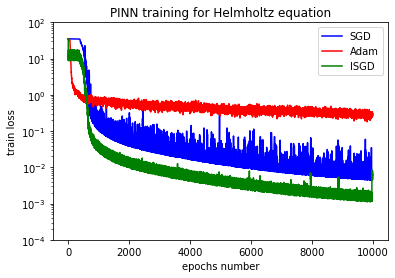}}
     \subfloat[learning rate = 0.5]{\includegraphics[width=0.45\columnwidth]{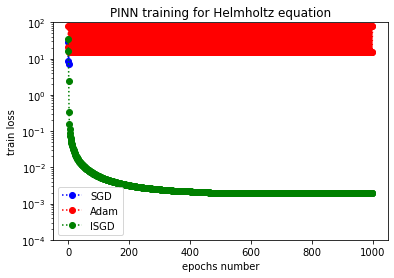}}\\
     \subfloat[prediction $u_{\theta}(x,y)$]{\includegraphics[width=0.45\columnwidth]{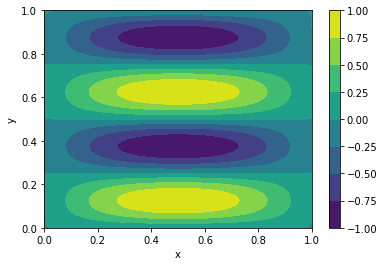}}\
     \subfloat[error $|u(x,y)-u_{\theta}(x,y)|$]{\includegraphics[width=0.45\columnwidth]{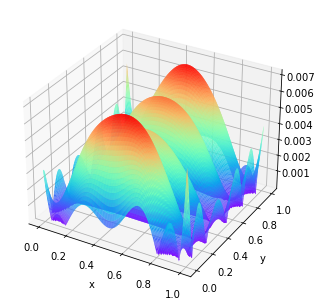}}
     \caption{PINN training for Helmholtz equation \eqref{ex2-eq3}.}\label{fig:ex2-4}
\end{figure}
The Helmholtz equation is one of the fundamental equations of mathematical physics
arising in many physical problems, such as vibrating membranes, acoustics, and electromagnetism
equations.
We solve the two-dimensional Helmholtz equation given by
\bea\label{ex2-eq3}
\left\{
  \begin{array}{l}
  \frac{\p^2u}{\p x^2}+\frac{\p^2u}{\p y^2} +k^2u(x,y) = f(x,y),\quad (x,y)\in\Og, \\
  u(x,y) = 0, \quad (x,y)\in\p\Og.
  \end{array}
\right.
\eea
The exact solution for $k=4$ is $u(x,y)=\sin(\pi x)\sin(4\pi y)$, and the force term $f(x,y)$ is given by the Eq.\eqref{ex2-eq3}.
We choose $N_b=400$ randomly sampled points on $\p\Og$,
and $N_f=4,000$ randomly sampled points in $\Og$ to compute the loss function.
A neural network with 6 hidden layers, every 100 units with \emph{tanh} activations,
is applied in all the computations.
The three optimizer training results for $\alpha=0.0005$ and $0.5$ are shown in Fig. \ref{fig:ex2-4}(a) and Fig. \ref{fig:ex2-4}(b), respectively.
%We see that neither SGD nor Adam  trains well, but our ISGD trains well for different values of $\alpha$.
The PINN solution is plotted in Fig. \ref{fig:ex2-4}(c),
and the absolute error is shown in Fig. \ref{fig:ex2-4}(d), with an absolute error less than $0.7\%$.
We see that the PINN trained by the ISGD optimizer can obtain stable and accurate results for the Helmholtz equation \eqref{ex2-eq3}.

\section{Conclusion}
To overcome the numerical instability of traditional gradient descent methods to some key hyper-parameters,
a stable IGD/ISGD method was proposed, analyzed and tested in this paper.
The IGD/ISGD method includes implicit updates, and the L-BFGS or Adam optimizer can be combined to forward the updates.
The global convergence of IGD/ISGD are theoretically analyzed and proven.
We apply the IGD/ISGD method to train deep as well as physics-informed
neural networks, showing that the IGD/ISGD method can effectively deal with stiffness phenomenon in the training dynamics via gradient descent.
The techniques proposed in this paper stabilize the training of neural network models.
This may result in making it easier for non-experts to train such models for beneficial applications, such as solving PDEs.

\section*{Acknowledgments}
The first author is supported by the National Natural Science Foundation of China (No.62106103), Fundamental Research Funds for the Central Universities (No.ILA22023) and 173 Program Technical Field Fund (No.2021-JCJQ-JJ-0018).

\bibliography{isgd_aaai23}

%%%%%%%%%%%%%%%%%%%%%%%%%%%%%%%%%%%%%%%%%%%%%%%%%%%%%%%%%%%%%%%%%%%%%%%%%%%%%%%
%%%%%%%%%%%%%%%%%%%%%%%%%%%%%%%%%%%%%%%%%%%%%%%%%%%%%%%%%%%%%%%%%%%%%%%%%%%%%%%
% APPENDIX
%%%%%%%%%%%%%%%%%%%%%%%%%%%%%%%%%%%%%%%%%%%%%%%%%%%%%%%%%%%%%%%%%%%%%%%%%%%%%%%
%%%%%%%%%%%%%%%%%%%%%%%%%%%%%%%%%%%%%%%%%%%%%%%%%%%%%%%%%%%%%%%%%%%%%%%%%%%%%%%
\newpage
\appendix
\onecolumn
\section{IGD overcomes stiffness in the gradient flow dynamics}
The following analysis reveals how IGD can overcome the stiffness in the gradient flow dynamics of PINNs.
Suppose that at $n$-th step of IGD during training of minimizing the loss $L(\theta)$, we have
\begin{equation}\label{app:c1}
  \theta_{n+1} = \theta_{n} - \alpha \cdot \nabla_{\theta} L(\theta_{n+1}),
\end{equation}
or equally
\begin{equation}\label{app:c2}
  \theta_{n} = \theta_{n+1} + \alpha \cdot \nabla_{\theta} L(\theta_{n+1}),
\end{equation}
where $\alpha$ is the learning rate.
We will show that $L(\theta_{n+1})-L(\theta_n)<0$ regardless of the stiffness in the gradient flow dynamics of PINNs, provided that $\theta_n$ approaches a local or global minimum of $L(\theta)$.

Applying second order Taylor expansion to the loss function $L(\theta)$ at $\theta_{n+1}$ gives
\be\label{app:c3}
L(\theta_n) = L(\theta_{n+1}) + (\theta_n-\theta_{n+1})\cdot \nabla_{\theta} L(\theta_{n+1})
+\frac{1}{2}(\theta_n-\theta_{n+1})^T \nabla_{\theta}^2 L(\xi) (\theta_n-\theta_{n+1}),
\ee
where $\xi=t\theta_{n+1}+(1-t)\theta_n$ for some $t\in(0,1)$, and $\nabla_{\theta}^2 L(\xi)$ is the Hessian matrix of the loss function $L(\theta)$ evaluated at the point $\theta=\xi$.
Now applying \eqref{app:c2} to \eqref{app:c3}, we obtain
\beas
L(\theta_{n+1}) - L(\theta_n) = -\alpha\norm{\nabla_{\theta} L(\theta_{n+1})}_2^2 -\frac{1}{2}\alpha^2\nabla_{\theta} L(\theta_{n+1})^T\nabla_{\theta}^2 L(\xi)\nabla_{\theta} L(\theta_{n+1}).
\eeas
Denote the normalized vector $x=\frac{\nabla_{\theta} L(\theta_{n+1})}{\norm{\nabla_{\theta} L(\theta_{n+1})}}$, and $Q$ is an orthogonal matrix diagonalizing $\nabla_{\theta}^2 L(\xi)$ with $\nabla_{\theta}^2 L(\xi)=Q^T\text{diag}\{\lambda_1,\lambda_2,...,\lambda_N\}Q$, so $y=Qx$ is a normalized vector and we obtain
\bea
\nabla_{\theta} L(\theta_{n+1})^T\nabla_{\theta}^2 L(\xi)\nabla_{\theta} L(\theta_{n+1}) &=&
\norm{\nabla_{\theta} L(\theta_{n+1})}_2^2\cdot \frac{\nabla_{\theta} L(\theta_{n+1})^T}{\norm{\nabla_{\theta} L(\theta_{n+1})}}\nabla_{\theta}^2 L(\xi)\frac{\nabla_{\theta} L(\theta_{n+1})}{\norm{\nabla_{\theta} L(\theta_{n+1})}} \nonumber\\
&=& \norm{\nabla_{\theta} L(\theta_{n+1})}_2^2\cdot x^TQ^T\text{diag}\{\lambda_1,\lambda_2,...,\lambda_N\}Qx \nonumber\\
&=& \norm{\nabla_{\theta} L(\theta_{n+1})}_2^2\cdot y^T\text{diag}\{\lambda_1,\lambda_2,...,\lambda_N\}y \nonumber\\
&=& \norm{\nabla_{\theta} L(\theta_{n+1})}_2^2\sum_{i=1}^{N}\lambda_iy_i^2.
\eea
Combining these together we get
\begin{equation}
  L(\theta^{n+1})-L(\theta^n) = \alpha \norm{\nabla_{\theta}L(\theta^n)}_2^2
  \left(-1-\f{1}{2}\alpha\sum_{i=1}^{N}\lambda_iy_i^2\right).
\end{equation}
When $\theta_n$ approaches a local or global minimum of the loss function $L(\theta)$, we know that the Hessian matrix is semi-positive definite, i.e. $\lambda_i\geq0$ for all $i=1,2,...,N$.
This implies $L(\theta^{n+1})-L(\theta^n)<0$.
Moreover, when the gradient flow dynamics of PINNs is stiff, i.e., there exits at least one eigenvalues $\lambda_i$ that is very large, the loss may decay even faster compared to the non-stiff case.
This simple analysis illustrates the robustness of IGD optimizer in training PINNs with multi-frequency and multi-scale features.

\section{Proof of Theorem 1}\label{app3}
\begin{proof}
The proof technique is from \cite{Du19}'s global convergence proof of gradient descent for fully connected neural networks.
The proof sketch is as follows.
First, we show that the loss training dynamics are decreasing,
thus, the weights can be close to their initialization as long as $m$ is large enough.
This fact shows that the change in Gram matrix $\norm{\mH(k)-\mH(0)}_2$ is small.
The initial Gram matrix is strictly positive definite with large probability,
so the strict positivity of $\mH(k)$ is still guaranteed as $k$ increases.
This in turn sharpens the loss function's decreases.

The proof of Theorem 1 is conducted by induction.
It is easy to verify that the condition holds for $k^{'}=0$ and assume that it holds for $k^{'}=k$.
Now for $k^{'}=k+1$, we have
\be\label{app:eq1}
\y-\u(k+1) = (1+\alpha\mH(k+1))^{-1}[\y-\u(k)+\mathbf{I}_0(k+1)].
\ee
Denote $\lambda_0 = \lambda_{\min}(\mH^{\infty})$.
From Lemma \ref{lem1} we know $\lambda_0>0$.
From Lemma \ref{lem4} we have $\lambda_{\min}(\mH(k+1))\geq\f{\lambda_0}{2}$, so
\be
\norm{(1+\alpha\mH(k+1))^{-1}}_2 \leq \f{1}{1+\f{\alpha\lambda_0}{2}},\nonumber
\ee
then we estimate Eq.\eqref{app:eq1} by
\bea
&& \norm{\y-\u(k+1)}_2^2 \nonumber\\
&\leq& \norm{(1+\alpha\mH(k+1))^{-1}}_2^2\left( \norm{\y-\u(k)}_2^2 +
2\inn{\y-\u(k),\mathbf{I}_0(k+1)} + \norm{\mathbf{I}_0(k+1)}_2^2\right)  \nonumber\\
&\leq& \left(\f{1}{1+\f{\alpha\lambda_0}{2}}\right)^2\left( \norm{\y-\u(k)}_2^2 +
2\inn{\y-\u(k),\mathbf{I}_0(k+1)} + \norm{\mathbf{I}_0(k+1)}_2^2\right). \nonumber
\eea
From Lemma \ref{lem5} we have
\bea
\inn{\y-\u(k),\mathbf{I}_0(k+1)} &\leq&
\f{\alpha\lambda_0}{8}\norm{\y-\u(k)}_2\norm{\y-\u(k+1)} \nonumber\\
&\leq& \f{\alpha\lambda_0}{16}\norm{\y-\u(k)}_2^2 +
       \f{\alpha\lambda_0}{16}\norm{\y-\u(k+1)}_2^2, \nonumber\\
\norm{\mathbf{I}_0(k+1)}_2^2 &\leq& \f{\alpha\lambda_0}{8}\norm{\y-\u(k+1)}_2^2. \nonumber
\eea
Subtracting we obtain
\bea
\norm{\y-\u(k+1)}_2^2 &\leq& \f{1+\f{\alpha\lambda_0}{8}}
{\left(1+\f{\alpha\lambda_0}{2}\right)^2-\f{\alpha\lambda_0}{4}} \norm{\y-\u(k)}_2^2 \nonumber\\
&\leq& \f{1}{1+\f{\alpha\lambda_0}{2}}\norm{\y-\u(k)}_2^2 \nonumber\\
&\leq& \left(\f{1}{1+\f{\alpha\lambda_0}{2}}\right)^{k+1}\norm{\y-\u(0)}_2^2. \nonumber
\eea
Noting that $L(k+1)=\f{1}{2}\norm{\y-\u(k+1)}_2^2$ we finish the proof.
\end{proof}

\subsection{Some Lemmas in the proof of Theorem 1}
The following two lemmas can be found in \citet{Du19}.
Lemma \ref{lem1} is Lemma F.2 in \citet{Du19}, and Lemma \ref{lem2} is Lemma B.2 in \citet{Du19} in a simple form.
For self-consistency, we also give the proof simply.
\begin{lemma}\label{lem1}
  Under Assumption 2, the Gram matrix is strictly positive definite, i.e.,
  $\lambda_0\triangleq \lambda_{\min}(\mH^{\infty}) >0$.
\end{lemma}
\begin{proof}
  %By the definition of $\mH^{\infty}$,
  %We get the feature map $\phi_{\x}(\w)=\x\sigma^{'}(\w^T\x)$ induced by the kernel $G$.
  To show $\mH^{\infty}$ is strictly positive definite,
  it is equivalent to showing that for $\v\in\mathbb{R}^d$, $\v^T\mH^{\infty}\v=0$ implies $\v=\mathbf{0}$. That is,
  \be
  0=\sum_{i,j=1}^{n}v_iv_j\x_i^T\x_j\mathbb{E}_{\w\sim\mathcal{N}(\mathbf{0},\mathbf{I})}
  \sigma^{'}(\w^T\x_i)\sigma^{'}(\w^T\x_j).
  \ee
  We obtain
  \be
  0 = \sum_{i=1}^{n}v_i\phi_{\x_i}(\w) = \sum_{i=1}^{n}v_i\sigma^{'}(\w^T\x_i)\x_i,
  \quad a.e. \;\w\in\mathbb{R}^d.
  \ee
  Differentiating the above equation $n-1$ times with respect to $\w$, we have
  \be
  0 = \sum_{i=1}^{n}v_i\sigma^{(n)}(\w^T\x_i)\x_i^{\otimes(n)},
  \quad a.e. \;\w\in\mathbb{R}^d.
  \ee
  From Lemma G.6 in \citet{Du19}, we know that $\{\x_i^{\otimes(n)}\}_{i=1}^n$ are linearly independent under Assumption 2.
  Therefore, we have $v_i\sigma^{(n)}(\w^T\x_i)=0$ for all $i\in[n]$ and $\w\in\mathbb{R}^d$.
  Choosing $\w$ such that $\sigma^{(n)}(\w^T\x_i)\neq0$, we obtain $v_i=0$ for all $i\in[n]$.
  So we complete the proof.
\end{proof}

\begin{lemma}\label{lem2}
  If $m\geq \f{16Cn^2}{\lambda_0^2}\log(\f{2n}{\delta})$ for some constant $C$,
  then with probability $1-\delta$ we have
\be
\lambda_{\min}(\mH(0)) \geq \f{3\lambda_0}{4}.
\ee
\end{lemma}
\begin{proof}
  This can be proven by the standard concentration technique.
  For every fixed $(i,j)$ pair, $\mH_{ij}(0)$ can be considered the average of the independent random variables. From the Hoeffding inequality we have with probability $1-\delta^{'}$
  \be
  |\mH_{ij}(0)-\mH_{ij}^{\infty}| \leq \sqrt{\f{C\log(\f{2}{\delta^{'}})}{2m}}.\nonumber
  \ee
  Setting $\delta^{'}=n^2\delta$ and applying union bound over $(i,j)$ pairs, we have for all $(i,j)$ pairs with probability at least $1-\delta$
  \be
  |\mH_{ij}(0)-\mH_{ij}^{\infty}| \leq \sqrt{\f{C\log(\f{2n}{\delta})}{m}}. \nonumber
  \ee
  Thus, we have
  \be
  \norm{\mH(0)-\mH^{\infty}}_2 \leq \norm{\mH(0)-\mH^{\infty}}_F =
  \sqrt{\sum_{i,j=1}^{n}|\mH_{ij}(0)-\mH_{ij}^{\infty}|^2} \leq
  \sqrt{\f{Cn^2\log(\f{2n}{\delta})}{m}}
  \leq \f{\lambda_0}{4}.\nonumber
  \ee
  So we obtain $\lambda_{\min}(\mH(0))\geq \lambda_{\min}(\mH^{\infty}) - \norm{\mH(0)-\mH^{\infty}}_2\geq \f{3\lambda_0}{4}$.
\end{proof}

The following lemma shows that if the induction holds, we have every weight vector close to its initialization.
\begin{lemma}\label{lem3}
If $L(k^{'}) \leq \left(1+\f{\alpha\lambda_{\min}(\mH^{\infty})}{2}\right)^{-k^{'}} L(0)$ holds for $k^{'}=1,2,...,k$, assume $\alpha\leq\f{C\lambda_0}{n^2}$,
we have for $s=1,2,...,k,k+1$
\bea
\norm{\w_r(s)-\w_r(s-1)}_2 &\leq& \f{C\alpha\sqrt{n}}{\sqrt{m}}\norm{\y-\u(s)}_2,\\
\norm{\w_r(s)-\w_r(0)}_2   &\leq& \f{2Cn}{\lambda_0\sqrt{m}}.
\eea
\end{lemma}
\begin{proof}
The training dynamics have the following relation for $s=1,2,...,k,k+1$
\be
\y-\u(s) = (1+\alpha\mH(s))^{-1}[\y-\u(s-1)+\mathbf{I}_0(s)],\nonumber
\ee
so
\bea
\norm{\y-\u(s)}_2 &\leq& \norm{(1+\alpha\mH(s))^{-1}}_2[\norm{\y-\u(s-1)}_2+\norm{\mathbf{I}_0(s)}_2] \nonumber\\
&\leq& \left(\f{1}{1+\f{\alpha\lambda_0}{2}}\right)^{\f{s-1}{2}}\norm{\y-\u(0)}_2 + Cn\alpha^2.
\eea
The weights change
\be
\w_r(s)-\w_r(s-1) = -\alpha\sum_{j=1}^{n}(y_j-u_j(s))\f{1}{\sqrt{m}}a_r\x_j\sigma^{'}(\w_r(s)^T\x_j),\nonumber
\ee
so
\bea
\norm{\w_r(s)-\w_r(s-1)}_2 &\leq& \f{C\alpha}{\sqrt{m}}\sum_{j=1}^{n}|y_j-u_j(s)| \nonumber\\
&\leq& \f{C\alpha\sqrt{n}}{\sqrt{m}}\norm{\y-\u(s)}_2.\nonumber
\eea
Similarly, we can bound for $s=1,2,\cdots,k$
\bea
\norm{\w_r(s)-\w_r(0)}_2 &\leq& \f{C\alpha\sqrt{n}}{\sqrt{m}}\sum_{s^{'}=1}^{s}\norm{\y-\u(s^{'})}_2 \nonumber\\
&\leq& \f{C\alpha\sqrt{n}}{\sqrt{m}}\sum_{s^{'}=1}^{s}\left(\f{1}{1+\f{\alpha\lambda_0}{2}}\right)^{\f{s^{'}}{2}}\norm{\y-\u(0)}_2 \nonumber\\
&\leq& \f{C\alpha\sqrt{n}}{\sqrt{m}}\sum_{s^{'}=1}^{\infty}\left(\f{1}{1+\f{\alpha\lambda_0}{2}}\right)^{\f{s^{'}}{2}}\norm{\y-\u(0)}_2 \nonumber\\
&\leq& \f{Cn}{\sqrt{m}\lambda_0},
\eea
and for $s=k+1$
\bea
\norm{\w_r(k+1)-\w_r(k)}_2 &\leq& \f{C\alpha\sqrt{n}}{\sqrt{m}}\norm{\y-\u(k+1)}_2 \nonumber\\
&\leq& \f{C\alpha\sqrt{n}}{\sqrt{m}}\left(\f{\norm{\y-\u(0)}_2}{(1+\f{\alpha\lambda_0}{2})^{\f{k}{2}}}+n\alpha^2\right) \nonumber\\
&\leq& \f{Cn}{\sqrt{m}\lambda_0}.
\eea
Hence, we have
\be
\norm{\w_r(k+1)-\w_r(0)}_2 \leq \f{2Cn}{\sqrt{m}\lambda_0}.
\ee
and we finish the proof.
\end{proof}

\begin{lemma}\label{lem4}
If $L(k^{'}) \leq \left(1+\f{\alpha\lambda_{\min}(\mH^{\infty})}{2}\right)^{-k^{'}} L(0)$ holds for $k^{'}=1,2,...,k$, assume $\alpha\leq\f{C\lambda_0}{n^2}$
and $m\geq\f{16C^2n^4}{\lambda_0^4}$, then we have for $s=1,2,...,k,k+1$
\bea
\norm{\mH(s)-\mH(0)}_2 &\leq& \f{\lambda_0}{4},\\
\lambda_{\min}(\mH(s)) &\geq& \f{\lambda_0}{2}.
\eea
\end{lemma}
\begin{proof}
  We calculate the difference between $\mH(s)$ and $\mH(0)$
  \bea
  \mH_{ij}(s)-\mH_{ij}(0) &=& \f{1}{m}\sum_{r=1}^{m} \x_i^T\x_j\sigma^{'}(\w_r(s)^T\x_i)\left[\sigma^{'}(\w_r(s)^T\x_j)-\sigma^{'}(\w_r(0)^T\x_j)\right] \nonumber\\
  & & +\f{1}{m}\sum_{r=1}^{m} \x_i^T\x_j\left[\sigma^{'}(\w_r(s)^T\x_i)-\sigma^{'}(\w_r(0)^T\x_i)\right]\sigma^{'}(\w_r(0)^T\x_j), \nonumber
  \eea
  and found it bounded by the weight difference
  \bea
  |\mH_{ij}(s)-\mH_{ij}(0)| \leq \f{C}{m}\sum_{r=1}^{m}\norm{\w_r(s)-\w_r(0)}_2,
  \eea
  so we obtain
  \be
  \norm{\mH(s)-\mH(0)}_2 \leq \norm{\mH(s)-\mH(0)}_F = \sqrt{\sum_{i,j=1}^{n}|\mH_{ij}(s)-\mH_{ij}(0)|^2} \leq \f{Cn^2}{\sqrt{m}\lambda_0} \leq \f{\lambda_0}{4},\nonumber
  \ee
  and $\lambda_{\min}(\mH(s))\geq \lambda_{\min}(\mH(0)) - \norm{\mH(s)-\mH(0)}_2\geq \f{\lambda_0}{2}$.
\end{proof}

\begin{lemma}\label{lem5}
If $L(k^{'}) \leq \left(1+\f{\alpha\lambda_{\min}(\mH^{\infty})}{2}\right)^{-k^{'}} L(0)$ holds for $k^{'}=1,2,...,k$, assume $\alpha\leq\f{C\lambda_0}{n^2}$,
then we have for $s=1,2,...,k,k+1$
\bea
\norm{\mathbf{I_0(s)}}_2 &\leq& \f{\alpha\lambda_0}{8}\norm{\y-\u(s)}_2,\\
\norm{\mathbf{I_0(s)}}_2^2 &\leq& \f{\alpha\lambda_0}{8}\norm{\y-\u(s)}_2^2.
\eea
\end{lemma}
\begin{proof}
  For fixed $i\in[n]$ we have
  \be
  I_0^i(s) = \sum_{r=1}^{m}\int_{0}^{\alpha}\inn{\f{\p u_i(\W(s)+\eta\f{\p L(s)}{\p \w_r})}{\p \w_r}-\f{\p u_i(\W(s))}{\p \w_r},\f{\p L(s)}{\p \w_r}}d\eta, \nonumber
  \ee
  so we bound it by
  \bea
  |I_0^i(s)| &\leq& \sum_{r=1}^{m}\alpha^2\norm{\f{\p L(s)}{\p \w_r}}_2^2|u_i^{(2)}(\xi(s))| \nonumber\\
  &\leq& C \alpha^2\sum_{r=1}^{m}\norm{\f{\p L(s)}{\p \w_r}}_2^2. \nonumber
  \eea
  Note that
  \bea
  \norm{\f{\p L(s)}{\p \w_r}}_2 &=& \norm{\sum_{i=1}^{n}(y_i-u_i(s))\f{1}{\sqrt{m}}\x_i\sigma^{'}(\w_r(s)^T\x_i)}_2 \nonumber\\
  &\leq& \f{C}{\sqrt{m}}\sum_{i=1}^{n}|y_i-u_i(s)| \nonumber\\
  &\leq& \f{C\sqrt{n}}{\sqrt{m}}\norm{\y-\u(s)}_2, \nonumber
  \eea
  we have
  \be
  \norm{\mathbf{I}_0(s)}_2 \leq C\alpha^2n\norm{\y-\u(s)}_2^2. \nonumber
  \ee
  Since $\alpha\leq \f{C\lambda_0}{n^2}$ and noting that $\norm{\y-\u(0)}_2=\mathcal{O}(\sqrt{n})$,
  we have
  \bea
  \norm{\mathbf{I_0(s)}}_2 &\leq& \f{\alpha\lambda_0}{8}\norm{\y-\u(s)}_2,\\
  \norm{\mathbf{I_0(s)}}_2^2 &\leq& \f{\alpha\lambda_0}{8}\norm{\y-\u(s)}_2^2.
  \eea
\end{proof}

\section{Additional Computational Results}
\subsection{Neural network approximation of nonlinear multiscale and discontinuous functions}
In this test case, we use the standard neural network to approximate given functions.
Theoretically speaking, neural networks can approximate any continuous function in some function spaces.
However, training the neural network to approximate given functions is nontrivial,
especially for functions with multiscale phenomena or even discontinuous phenomena.

First, we consider a function with multiscale phenomenon
\be\label{ex1-eq1}
u(x) = (x^3-x)\f{\sin(4x)}{4} + \f{\sin(12x)}{x^2+1},\quad x\in[-3,3].
\ee
Second, we consider a function with discontinuity
\be\label{ex1-eq2}
u(x) = \left\{
\begin{array}{ll}
  \sin(4x) &  x \in [-3,0]  \\
  2 + x\sin(x) & x \in (0,3].
\end{array}\right.
\ee
The activation function is $tanh$, and the number of hidden layers is 4 with 50 neurons in each layer.

A recent study by \citet{RBA19} showed that neural network training suffers from the ``spectral bias" phenomenon,
which means that neural networks learn low frequencies first,
then learn high frequencies at a very slow rate.
Figure \ref{fig:ex1-1} shows the training loss and the predicted solution training by SGD, Adam and ISGD optimizers for the neural network approximation of the multiscale function Eq.\eqref{ex1-eq1}.
For a small learning rate $\alpha=0.005$, ISGD is very close to SGD,
and all three optimizers need many epochs (up to 350 K for ISGD) to learn both the low- and high-frequency components of Eq.\eqref{ex1-eq1}.
%We note that ISGD stops in small local minima after 350k epochs, while SGD can escape from local minima, which is consistent with the implicit regularizer result in Theorem \ref{thm1} and Theorem \ref{thm2}.
As the learning rate increases to $\alpha=0.05$, neither SGD nor Adam can learn the high-frequency components in smaller epochs,
while our ISGD method can capture all frequencies of Eq.\eqref{ex1-eq1} in 35k epochs.
When a larger learning rate applies, for example, $\alpha=0.5$ and $2.5$,
the SGD and Adam may not be convergent or even explode because of numerical instability,
while our ISGD method can still capture all frequencies of Eq.\eqref{ex1-eq1} in smaller epochs.

Figure \ref{fig:ex1-2} shows the training loss and the predicted solution training by SGD, Adam and ISGD optimizers for the neural network approximation of the discontinuous function Eq.\eqref{ex1-eq2}.
We see that for small learning rate $\alpha=0.005$ and $0.05$,
both three optimizers can approximate the function well, except for some oscillations near the discontinuity point $x=0$.
When a larger learning rate applies, for example, $\alpha=0.5$ and $2.5$,
the SGD and Adam may not be convergent or even explode because of numerical instability,
while our ISGD method can still approximate the discontinuity of Eq.\eqref{ex1-eq2} in smaller epochs.

\begin{figure}[tbp]
   \centering
     \subfloat[learning rate = 0.005]{\includegraphics[width=0.38\textwidth]{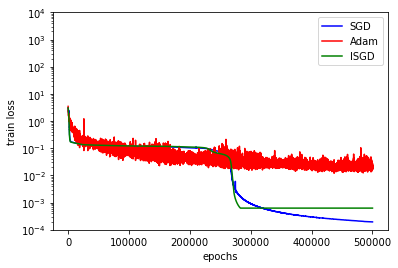}}
     \subfloat[learning rate = 0.005]{\includegraphics[width=0.38\textwidth]{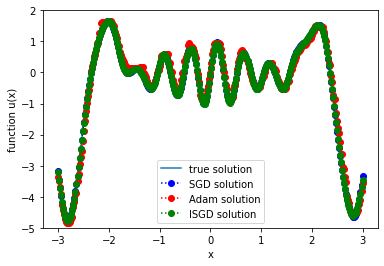}}\\
     \subfloat[learning rate = 0.05]{\includegraphics[width=0.38\textwidth]{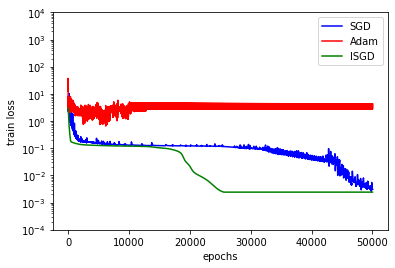}}
     \subfloat[learning rate = 0.05]{\includegraphics[width=0.38\textwidth]{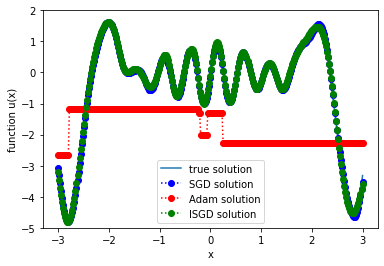}}\\
     \subfloat[learning rate = 0.5]{\includegraphics[width=0.38\textwidth]{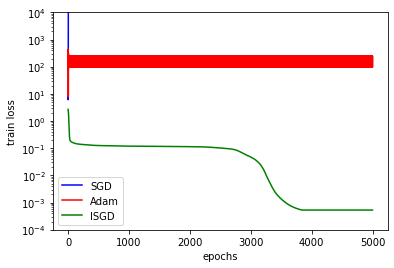}}
     \subfloat[learning rate = 0.5]{\includegraphics[width=0.38\textwidth]{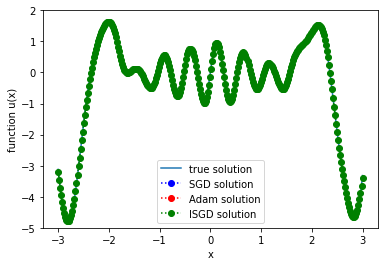}}\\
     \subfloat[learning rate = 2.5]{\includegraphics[width=0.38\textwidth]{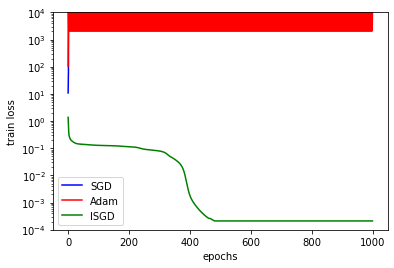}}
     \subfloat[learning rate = 2.5]{\includegraphics[width=0.38\textwidth]{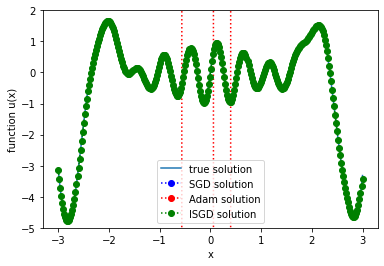}}
     \caption{The optimization training results for multiscale function approximation Eq.\eqref{ex1-eq1}. Left column: the training loss dynamics by three optimizers for learning rate = 0.005, 0.05, 0.5 and 2.5 (from top to bottom). Right column: the predicted solution trained by three optimizers for learning rate = 0.005,0.05,0.5 and 2.5(from top to bottom).}\label{fig:ex1-1}
\end{figure}

\begin{figure}[tbp]
   \centering
     \subfloat[learning rate = 0.005]{\includegraphics[width=0.38\textwidth]{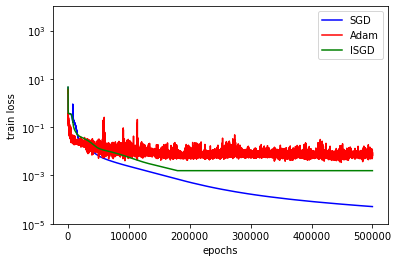}}
     \subfloat[learning rate = 0.005]{\includegraphics[width=0.38\textwidth]{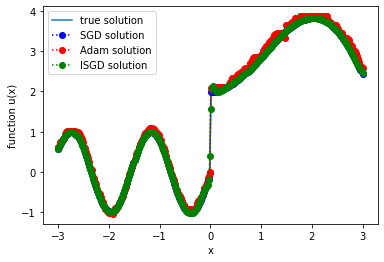}}\\
     \subfloat[learning rate = 0.05]{\includegraphics[width=0.38\textwidth]{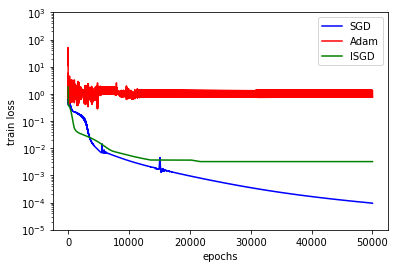}}
     \subfloat[learning rate = 0.05]{\includegraphics[width=0.38\textwidth]{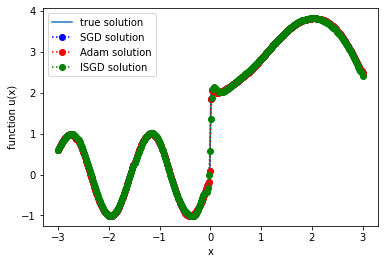}}\\
     \subfloat[learning rate = 0.5]{\includegraphics[width=0.38\textwidth]{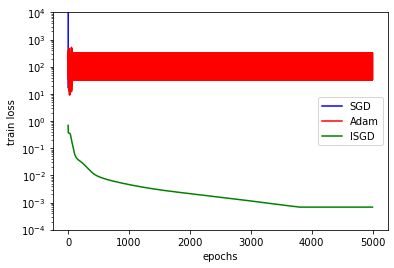}}
     \subfloat[learning rate = 0.5]{\includegraphics[width=0.38\textwidth]{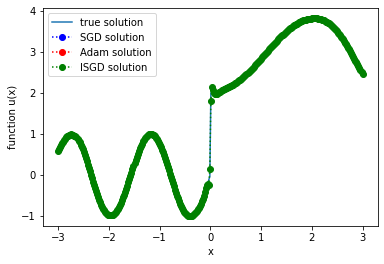}}\\
     \subfloat[learning rate = 2.5]{\includegraphics[width=0.38\textwidth]{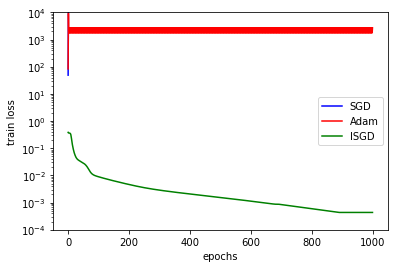}}
     \subfloat[learning rate = 2.5]{\includegraphics[width=0.38\textwidth]{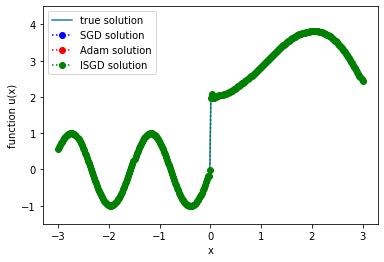}}
     \caption{The optimization training results for discontinuous function approximation Eq.\eqref{ex1-eq2}. Left column: the training loss dynamics by three optimizers for learning rate = 0.005, 0.05, 0.5 and 2.5 (from top to bottom). Right column: the predicted solution trained by three optimizers for learning rate = 0.005,0.05,0.5 and 2.5(from top to bottom).}\label{fig:ex1-2}
\end{figure}

\subsection{Standard deep learning benchmark problems}
The MNIST database is a database of handwritten digits that is commonly used for training various image processing systems.
It contains 60,000 training images and 10,000 testing images of dimension $24\times24=784$.
We train a two-layer neural network for the classification.
The neural network has a hidden layer with 128 units and \emph{ReLU} activations,
and an output layer with 10 units and \emph{Softmax} activations.

SGD, Adam optimizers and the ISGD optimizer proposed in this paper are compared for the neural network training.
In order to evaluate the training effect for different learning rates $\alpha$ and batch sizes $b$, we train fixed 10 epochs for batch size $b=32,128,512$ and fixed 100 epochs for the full batch $b=60,000$.
The test accuracy results are listed in Table \ref{tab1}.
Generally, the widely used SGD and Adam optimizers are both sensitive to the learning rates and batch sizes. They behave well for small learning rates and small batch sizes but may change dramatically for different learning rates and batch sizes and are especially unstable for large learning rates.
Conversely, the ISGD optimizer achieves the same level of accuracy as SGD and Adam for small learning rates and small batch sizes, and it can still be stable with high accuracy when the learning rate becomes very large.
This allows the nonexperts to train deep models in a more easier way.

\begin{table}[tbp]
\centering
\begin{tabular}{l|cccccc}\hline
 \multirow{2}*{\textbf{Optimizer}}  & \multicolumn{6}{c}{$b$ = 32}  \\\cline{2-7}
  & $\alpha=0.001$ & $\alpha=0.01$ & $\alpha=0.1$ & $\alpha=1.0$ & $\alpha=2.0$ & $\alpha=10.0$ \\\hline
 SGD  & 0.907 & 0.950 & \textbf{0.978} & 0.936 & 0.217 & 0.096  \\\hline
 Adam & \textbf{0.966} & \textbf{0.969} & 0.871 & 0.432 & 0.415 & 0.131  \\\hline
 ISGD & 0.903 & 0.950 & 0.975 & \textbf{0.978} & \textbf{0.973} & \textbf{0.971}  \\\hline\hline
 \multirow{2}*{\textbf{Optimizer}}  & \multicolumn{6}{c}{$b$ = 128}  \\\cline{2-7}
  & $\alpha=0.001$ & $\alpha=0.01$ & $\alpha=0.1$ & $\alpha=1.0$ & $\alpha=2.0$ & $\alpha=10.0$ \\\hline
 SGD  & 0.857 & 0.926 & 0.967 & 0.973 & 0.101 & 0.089 \\\hline
 Adam & \textbf{0.972} & \textbf{0.973} & 0.895 & 0.508 & 0.413 & 0.150 \\\hline
 ISGD & 0.853 & 0.924 & \textbf{0.968} & \textbf{0.978} & \textbf{0.977} & \textbf{0.967} \\\hline\hline
 \multirow{2}*{\textbf{Optimizer}}  & \multicolumn{6}{c}{$b$ = 512} \\\cline{2-7}
  & $\alpha=0.001$ & $\alpha=0.01$ & $\alpha=0.1$ & $\alpha=1.0$ & $\alpha=2.0$ & $\alpha=10.0$ \\\hline
 SGD  & 0.665 & 0.894 & 0.942 & \textbf{0.975} & 0.219 & 0.114  \\\hline
 Adam & \textbf{0.970} & \textbf{0.971} & 0.939 & 0.626 & 0.508 & 0.170  \\\hline
 ISGD & 0.691 & 0.895 & \textbf{0.943} & \textbf{0.975} & \textbf{0.976} & \textbf{0.969}  \\\hline\hline
 \multirow{2}*{\textbf{Optimizer}}  & \multicolumn{6}{c}{$b$ = full batch 60,000} \\\cline{2-7}
  & $\alpha=0.001$ & $\alpha=0.01$ & $\alpha=0.1$ & $\alpha=1.0$ & $\alpha=2.0$ & $\alpha=10.0$ \\\hline
 SGD  & 0.131 & 0.666 & 0.886 & \textbf{0.945} & 0.242 & 0.012 \\\hline
 Adam & \textbf{0.919} & \textbf{0.853} & 0.725 & 0.459 & 0.391 & 0.364 \\\hline
 ISGD & 0.141 & 0.626 & \textbf{0.887} & 0.941 & \textbf{0.952} & \textbf{0.976} \\\hline
\end{tabular}
\caption{Test accuracy comparison for different optimizers.}\label{tab1}
\end{table}

\end{document}